\documentclass[a4paper,11pt]{article}
\usepackage{booktabs} 
\usepackage{graphicx}
\usepackage{amsmath}
\usepackage{amssymb}
\usepackage[ruled,vlined]{algorithm2e}
\usepackage{amsthm}
\usepackage{amsfonts}
\usepackage{xspace}
\usepackage{pgfplots}
\usetikzlibrary{patterns}
\usepackage{enumitem}
\usepackage{mathtools}
\usepackage[noend]{algpseudocode}
\usepackage[toc,page]{appendix}
\pgfmathdeclarefunction{gauss}{2}{%
	\pgfmathparse{1/(#2*sqrt(2*pi))*exp(-((x-#1)^2)/(2*#2^2))}%
}
\providecommand{\keywords}[1]{\textbf{\textit{Index terms---}} #1}

\newcommand{\onemax}{\text{\sc OneMax}\xspace} 
\newcommand{\leadingones}{\text{\sc LeadingOnes}\xspace} 
\newcommand{\linear}{\text{\sc Linear}\xspace}

\newcommand{\bvleadingones}{\text{\sc BVLeadingOnes}\xspace} 
\newcommand{\umda}{\text{\sc UMDA}\xspace} 
\newcommand{\eda}{\text{\sc EDA}\xspace} 
\newcommand{\edas}{\text{\sc EDAs}\xspace} 
\newcommand{\ea}{\text{\sc EA}\xspace} 
\newcommand{\eas}{\text{\sc EAs}\xspace} 
\newcommand{\cga}{\text{\sc cGA}\xspace} 

\theoremstyle{definition}

\newtheorem{theorem}{Theorem}
\newtheorem{lemma}{Lemma}
\usepackage{graphicx}
\usepackage{amssymb}
\usepackage{lineno}

\setlist[itemize]{leftmargin=*}
\setlist[enumerate]{leftmargin=*}

\DeclareMathOperator{\bernoulli}{Ber} 
\DeclareMathOperator{\Var}{Var} 

\newcommand{\prob}[1]{\Pr\left(#1\right)}
\newcommand{\expect}[1]{\mathbb{E}\left[#1\right]}
\newcommand{\var}[1]{\Var\left[#1\right]}

\usepackage[numbers]{natbib}
\title{Improved Runtime Bounds for the Univariate 
Marginal Distribution Algorithm via 
Anti-Concentration\footnote{An extended abstract of this report 
appeared in the proceedings of the 2017 Genetic and 
Evolutionary Computation Conference (GECCO 2017).}}

\author{Per Kristian Lehre \& Phan Trung Hai Nguyen\\
School of Computer Science\\
University of Birmingham\\
Birmingham B15 2TT, UK\\
\texttt{\{p.k.lehre,p.nguyen\}@cs.bham.ac.uk}
}

\date{February 02, 2018}

\begin{document}
\maketitle

\begin{abstract}
Unlike traditional evolutionary algorithms 
which produce offspring via genetic operators,
Estimation of Distribution Algorithms (\eda{}s) sample solutions from
probabilistic models which are learned from selected individuals. 
It is hoped that \eda{}s may improve optimisation performance on
epistatic fitness landscapes by learning variable interactions.
However, hardly any rigorous results are available to support claims
about the performance of EDAs, even for fitness functions without
epistasis.  The expected runtime of the Univariate Marginal Distribution
Algorithm (UMDA) on \onemax was recently shown to be in
$\mathcal{O}\left(n\lambda\log \lambda\right)$
\cite{bib:Dang2015a}. Later, Krejca and Witt \cite{bib:Krejca}
proved the lower bound
$\Omega\left(\lambda\sqrt{n}+n\log n\right)$ via an involved drift
analysis . 

We prove a $O\left(n\lambda\right)$ bound, given some restrictions on
the population size. This implies the tight bound $\Theta\left(n\log
n\right)$ when $\lambda=O\left(\log n\right)$, matching the runtime of
classical EAs.  Our analysis uses the level-based theorem and
anti-concentration properties of the Poisson-binomial distribution. We
expect that these generic methods will facilitate further analysis of
EDAs.
\end{abstract}


\section{Introduction}
Estimation of Distribution Algorithms are a class 
of randomised search heuristics with many practical applications
\cite{bib:Hauschild}. Unlike traditional \eas which look for optimal
solutions by explicitly building and maintaining a population of
promising individuals, \edas rely on a probabilistic model to
represent information gained from the optimisation process over
generations. There are many different variants of \edas have been 
developed over the last decades, and the fundamental differences 
between them are the ways the interactions of decision variables are captured 
as well as how the probabilistic model is updated over generations. The earliest 
\edas treated each variable independently, whereas later ones 
model variable dependencies \cite{bib:Shapiro2005}. 
Some examples of univariate
\edas are the compact genetic algorithm (\cga) and the Univariate 
Marginal Distribution Algorithm (\umda). Multi-variate \edas, such as the
Bayesian Optimisation Algorithms  which builds a 
Bayesian network with nodes and edges representing 
variables and conditional dependencies, 
attempt to learn relationships between the
decision variables   \cite{bib:Hauschild}. 
See \cite{bib:Hauschild} 
for other variants and more practical applications of \edas.

The compact genetic algorithm was the first 
univariate \eda whose runtime was 
analysed rigorously. Introduced in \cite{bib:Harik}, 
the algorithm samples two individuals in each generation and then 
evaluates them to  determine the winner which is used to update 
the probabilistic model. A quantity of $1/K$ is shifted towards 
the winning bit value for each position where the two 
individuals differ. 
The first rigorous runtime analysis of \cga was completed by Droste in
\cite{bib:Droste2006} where a  lower bound  $\Omega(K\sqrt{n})$ 
for any functions is provided using additive drift theory 
where $n$ being the problem size. The result is obtained 
by estimating an upper bound for an entity named 
\textit{surplus} which is believed to reduce the overall running 
time if a large value appears in every generation. 
In addition, he proved an upper bound $\mathcal{O}(nK)$ for any 
linear function where $K=n^{1+\epsilon}$ 
for any small constant $\varepsilon>0$. Later studies showed
that given a fitness function $f$, \cga have problems
optimising functions with many 
$f$-independent bit positions, such as  \leadingones 
\cite{bib:Friedrich2016}. This is because the 
marginal probabilities of those positions are very 
close to the borders 0 or 1, which makes it harder 
to change those bits. A variant of the cGA, the so-called
stable compact genetic algorithm (s\cga) was introduced 
where the marginal probability $p_t(i)$ of any 
$f$-independent position tends to concentrate 
around $1/2$ (i.e. stable). Given certain parameter settings,
s\cga is able to optimise  \leadingones within 
$\mathcal{O}(n\log n)$ generations with probability polynomially close to $1$. 

Similar to \cga, \umda is a powerful algorithm 
with a wide range of applications not only in 
computer science but also in other areas.
The most studied variant is often implemented with upper 
and lower borders for marginal probabilities to prevent decision 
variables from being fixed at values zero or one. 
The population in each generation is sampled from a joint 
distribution which is the product of marginal probabilities 
for all variables. The UMDA is related to the notion of linkage
equilibrium, which is a popular assumption in Population Genetics.
Hence, understanding of \umda can contribute to the understanding of
population dynamics in Population Genetics models.

Despite the fact that the \umda has been analysed over the past years,
the understanding of its runtime is still limited. The algorithm was
analysed in series of papers \cite{bib:Chen2007,bib:Chen2009b,
	bib:Chen2009a,bib:Chen2010} where time-complexities of the \umda 
on simple unimodal functions were derived. These result shows 
that \umda with margins often outperforms 
other variants of \umda without margins, 
especially on functions like \bvleadingones. Shapiro 
\cite{bib:Shapiro2005} investigated \umda with a 
different selection mechanism rather than truncation selection. In
particular, their variant of \umda samples individuals whose
fitnesses are no less than the mean fitness before using them to
update the probabilistic model. By representing \umda as 
a Markov chain, the paper shows that the population size 
has to be in the order of square-root of the problem size for 
\umda to be able to optimise \onemax. The first upper 
bound on the expected optimisation time of \umda on 
\onemax was not published until 2015 
\cite{bib:Dang2015a}. By working on another variant 
of \umda which employs truncation selection,
Dang and Lehre \cite{bib:Dang2015a} proved an upper 
bound $\mathcal{O}(n\lambda\log \lambda)$ for \umda on \onemax 
which requires a population size $\Omega(\log n)$. 
If $\lambda = \mathcal{O}(\log n)$, then 
the upper bound is $\mathcal{O}(n\log n\log \log n)$.
The result is obtained by applying a relatively new 
technique called level-based theorem \cite{DBLP:Corus2017}. 
Very recently, Krejca and Witt \cite{bib:Krejca} obtain a 
lower bound $\Omega(\mu\sqrt{n}+n\log n)$ of \umda on 
\onemax via an involved drift analysis where 
$\lambda = (1+\Theta(1))\mu$. 
As can be seen, the upper and lower bounds 
are still different by $\Theta(\log \log n)$, 
which raises the question of whether this gap could be 
closed and a better asymptotic runtime would then be obtained.

This paper derives the
upper bound $\mathcal{O}(n\lambda)$ for \umda on \onemax which holds for 
$\lambda=\Omega(\mu)$ and $c\log n\leq \mu =\mathcal{O}(\sqrt{n})$ , 
where $c$ is some positive 
constant.
If $\lambda=\mathcal{O}(\log n)$, we have a tight bound 
$\Theta(n\log n)$ which matches with the well-known 
expected runtime $\Theta(n\log n)$ of the (1+1) \ea on the class of linear functions.
The result is achieved with the application of an anti-concentration 
bound which might be of general interest. The new result 
improves the known upper bound $\mathcal{O}(n\lambda\log \lambda)$ 
of \umda on \onemax \cite{bib:Dang2015a} by removing the logarithmic
factor $\mathcal{O}(\log \lambda)$. 
This improvement is significant becauses it for the first
time it closes the gap mentioned above for a 
small range of population size. In addition, 
we also believe that the easy-to-use method employed to 
obtain the result can be used for other algorithms and fitness functions.

This paper is structured as follows. In Section $2$, we first 
present the \umda algorithm under investigation. 
This section also includes a pseudo-code of the \umda.
The level-based theorem which is central in the paper will be 
stated in Section $3$. In this section, a sharp bound on the sum of 
Bernoulli random variables is also described. Given all 
necessary tools, Section $4$ illustrates our proof idea 
in a visual way and suggests how it could be applied for other problems.
The main result for
\umda on \onemax is presented in Section $5$. Section $6$ presents 
a brief empirical analysis of \umda on \onemax to complement 
the theoretical findings in Section $5$. Finally, concluding 
remarks are given in Section $7$.

\textit{Independent work:} Witt \cite{bib:Witt2017} 
independently obtained the upper bounds 
$\mathcal{O}(\mu n)$ on the expected optimisation 
time of the \umda on \onemax for $\mu \geq c\log n$, where $c$ is a positive constant, and $\lambda=(1+\Theta(1))\mu$
using an involved drift analysis. While our result does not hold for
$\mu=\omega(\sqrt{n})$, our methods yield a significantly easier proof
which also holds when the parent population size $\mu$ is not
proportional to the offspring population size $\lambda$.

\section{UMDA}
The Univariate Marginal Distribution Algorithm (UMDA) proposed in
\cite{bib:Muhlenbein1996} is one of the simplest variants of 
Estimation of Distribution Algorithms. In each
generation, the algorithm builds a probabilistic model over the
search space based on information gained about the individuals in the
previous generation. To optimise a pseudo-Boolean fitness function
$f:\{0,1\}^n\rightarrow\mathbb{R}$, the UMDA builds a product
distribution represented by a vector
$p_t=\left(p_t(1), p_t(2),\ldots,p_t(n)\right)$ in every generation
$t\in\mathbb{N}$.  Each component $p_t(i)\in[0,1]$ 
for $i\in[n]$ and $t\in \mathbb{N}$ represents
the probability of sampling a 1-bit at the $i$-th position of the
offspring in generation $t+1$ where $[n]$ denotes the set 
$\{1,2,3,\ldots,n\}$. Therefore, 
each candidate solution $(x_1,\ldots,x_n)\in \{0,1\}^n$ is 
sampled with joint probability
\begin{displaymath}
\Pr\left(x_1,\ldots,x_n\right)=\prod_{i=1}^{n}p_t(i)^{x_i}\cdot(1-p_t(i))^{(1-x_i)}.
\end{displaymath}
We will use the standard
initialisation $p_0(i)\coloneqq 1/2$ for all $i\in [n]$.
Starting with the initial model $p_0$, the algorithm 
continuously, in every generation $t\in\mathbb{N}$, sample
$\lambda$ individuals $P_t(1),\ldots,P_t(\lambda)$ 
using the current model $p_t$. All
individuals in the current population are sorted according to their
fitnesses, and the top $\mu$ individuals are selected to compute the
next model $p_{t+1}$. Let $P_t(k,i)$ denote the value in the $i$-th bit 
position of the $k$-th individual in current population $P_t$. Then
each component of the next model is defined as
\begin{displaymath}
p_{t+1}(i)\coloneqq \frac{1}{\mu}\sum_{k=1}^{\mu}P_t(k,i)
\end{displaymath}
which can be interpreted as the frequency of 1-bit 
among the $\mu$ best individuals in position $i$.

The special case $p_{t+1}(i)\in \{0,1\}$ must be avoided because the bit in
position $i$ would remain fixed forever at either 0 or 1. 
This would result in parts of the search space becoming unreachable.
In order to prevent this
situation, the model components are often restricted to a closed
interval, i.e. $p_{t+1}(i)\in [m'/\mu,1-m'/\mu]$, 
where the parameter $m'<\mu$  controls the size 
of the margins. For completeness, the following 
pseudo-code describes the full algorithm 
(see Algorithm \ref{umda-algor}).

\begin{algorithm}
	\DontPrintSemicolon	
	\Begin{
		initialise $p_0(i)=1/2$ for all $i \in [n]$,\;
		\For{$t=0,1,2,\ldots$ {\bf until termination condition met}}{
			\For{$k=1,2,\ldots,\lambda$}{
				sample $P_t(k,i) \sim \bernoulli\left(p_t(i)\right)$ for all $i\in[n]$\;
			}
			sort $P_t$ in descending order according to fitness,\;
			\For{$i=1,2,\ldots,n$}{
				let $X_i\coloneqq\sum_{k=1}^{\mu}P_t(k,i),$\;
				\uIf{$X_i<m'$}{$p_{t+1}(i)=m'/\mu$,\;}
				\uElseIf{$X_i>\mu-m'$}{	$p_{t+1}(i)=1-m'/\mu$, \;}
				\Else{$p_{t+1}(i)=X_i/\mu$. \;}					
	}}}
	\caption{\umda \label{umda-algor}}
\end{algorithm}

\section{Methods}
\subsection{Level-Based Theorem}\label{level-based-theorem-section}
The level-based theorem is a general tool that provides 
upper bounds on the expected optimisation time of 
many population-based algorithms on a wide range of 
optimisation problems. For example, it has been 
successfully applied to investigate the runtime of 
the Genetic Algorithms with or without crossover 
on various problems like $\linear$ or $\leadingones$ 
\cite{DBLP:Corus2017}. Besides, the first upper 
bounds of \umda on \onemax and \leadingones have 
been obtained using this method \cite{bib:Dang2015a}.

The theorem assumes that the algorithm to be analysed can be described in the form of 
Algorithm \ref{abstract-algor}. Let $\mathcal{X}$ be a finite 
search space which is, for example, $\{0,1\}^n$ in the case of 
binary representation. The algorithm considers a population 
$P_t$ at generation $t\in \mathbb{N}$ of 
$\lambda$ individuals that is represented as a vector 
$(P_t(1),P_t(2),\ldots,P_t(\lambda))\in \mathcal{X}^\lambda$. 
The theorem is general because it does not assume
specific fitness functions, selection mechanisms, or
generic operators like mutation and crossover. Rather, the theorem
assumes that there exists, possibly implicitly, a mapping 
$\mathcal{D}$ from the set of populations $\mathcal{X}^\lambda$ to the space of 
probability distribution over the search space  $\mathcal{X}$. The mapping $\mathcal{D}$ depends only 
on the current population and is used to produce the individuals in
the next generation \cite{DBLP:Corus2017}. 

\begin{algorithm}
	\DontPrintSemicolon
	\KwData{Finite search space $\mathcal{X}$, 
		population size $\lambda\in \mathbb{N}$, a\\
		mapping $\mathcal{D}$ from $\mathcal{X}^\lambda$ 
		to probability distributions over $\mathcal{X}$, and an
		initial population $P_0\in\mathcal{X}^\lambda.$}
	\Begin{
		\For{$t=0,1,2,\ldots$ {\bf until termination condition met}}{
			\For{$i=1,2,3,\ldots,\lambda$}{
				sample $P_{t+1}(i)\sim \mathcal{D}(P_t)$
	}}}
	\caption{Population-based algorithm\label{abstract-algor}}
\end{algorithm}

Furthermore, the theorem assumes a partition
$A_1,\ldots,A_m$ of the search space $\mathcal{X}$
into $m$ subsets, which we call levels. We assume that the last level $A_m$ 
consists of all optimal solutions. Although there are many 
different ways to create the partition, it should be chosen 
using prior knowledge of the specific problem under investigation and
the behaviour of the algorithm. 
One class of such partition is the well-known canonical 
fitness-based partition where all solutions with the same $f$-value are 
gathered to form a level. Let 
\begin{math}
A_{\geq j} \coloneqq \cup_{i=j}^{m}A_i
\end{math}
be the set of all individuals belonging to level 
$A_j$ or higher. We denote 
\begin{math}
|P_t \cap A_j|\coloneqq |\{i\mid P_t(i)\in A_j\}|
\end{math}
to be the number of individuals of the population $P_t$ belonging 
to level $A_j$. Given these conventions, we can state the level-based theorem as follows.
\begin{theorem}[Theorem 1, \cite{DBLP:Corus2017}]
	\label{thm:levelbasedtheorem}
	Given a partition $\left(A_i\right)_{i \in [m]}$ of $\mathcal{X}$, define 
	\begin{math}
	T\coloneqq\min\{t\lambda \mid |P_t\cap A_m|>0\}
	\end{math}	
	to be the first time $t$ that at least one element
	of level $A_m$ appears in the current 
	population $P_t$. If there exist 
	$z_1,\ldots,z_{m-1}, \delta \in (0,1]$, and $\gamma_0\in (0,1)$ 
	such that for any population $P_t \in \mathcal{X}^\lambda$,
	\begin{itemize}
		\item (G1) for each level $j\in[m-1]$, if $|P_t\cap A_{\geq j}|\geq \gamma_0\lambda$ then 
		\begin{displaymath}
		\Pr_{y \sim \mathcal{D}(P_t)}\left(y \in A_{\geq j+1}\right) \geq z_j.
		\end{displaymath}
		\item (G2) for each level $j\in[m-2]$, and all 
		$\gamma \in (0,\gamma_0]$, if $|P_t\cap A_{\geq j}|\geq 	
		\gamma_0\lambda$ and $|P_t\cap A_{\geq j+1}|\geq \gamma\lambda$ then
		\begin{displaymath}
		\Pr_{y \sim \mathcal{D}(P_t)}\left(y \in A_{\geq j+1}\right) \geq \left(1+\delta\right)\gamma.
		\end{displaymath}
		\item (G3) and the population size $\lambda \in \mathbb{N}$ satisfies
		\begin{displaymath}
		\lambda \geq \left(\frac{4}{\gamma_0\delta^2}\right)\ln\left(\frac{128m}{z_*\delta^2}\right) 
		\end{displaymath}
		where $z_* \coloneqq \min_{j\in [m-1]}\{z_j\}$, then
		\begin{displaymath}
		\mathbb{E}\left[T\right] \leq \left(\frac{8}{\delta^2}\right)\sum_{j=1}^{m-															1}\left[\lambda\ln\left(\frac{6\delta\lambda}{4+z_j\delta\lambda}\right)+\frac{1}{z_j}\right].
		\end{displaymath}
	\end{itemize}
\end{theorem}
Informally, the first condition (G1) requires that the 
probability to obtain an individual at level $A_{j+1}$ 
or higher is at least $z_j$ given that at least $\gamma_0\lambda$ 
individuals in the current population are in level $A_j$ 
or higher. Condition (G2) requires that given that $\gamma_0\lambda$ individuals 
of the current population belong to level $A_j$ or 
higher, and, moreover, $\gamma\lambda$ of them are lying at 
levels no lower than $A_{j+1}$, the probability of sampling a 
new offspring belonging to level $A_{j+1}$ or higher is no 
smaller than  $(1+\delta)\gamma$. The last condition (G3) 
sets a lower limit on the population size $\lambda$. 
As long as all three conditions are satisfied, 
an upper bound on the expected runtime of the population-based 
algorithm is guaranteed. 

Traditionally in Evolutionary Computation, we often 
define running time (or optimisation time) as the total 
number of fitness evaluations performed by the 
algorithm until an optimal solution has been found 
for the first time. However, 
the random variable $T\coloneqq\min\{t\lambda\mid|P_t\cap A_m|>0\}$ 
in Theorem~\ref{thm:levelbasedtheorem} 
is the total number of candidate 
solutions sampled by the algorithm until the first generation where an optimal 
solution is witnessed for the first time. In the 
context of \umda, these two entities are not always 
identical as $T$ is never smaller than the optimisation time.
Since the level-based theorem 
provides upper bounds on the optimisation time, this 
will not cause any problems.

The detailed proof of the level-based theorem can be seen in
\cite{DBLP:Corus2017} in which drift theory is applied to 
the distance measured by a level function. To apply 
the level-based theorem, it is recommended to 
follow a five-step procedure (see \cite{DBLP:Corus2017} 
for more details). It starts by identifying a proper partition of 
the search space, and then find specific parameter settings 
such that conditions (G1) and (G2) are met, followed by  
verifying that the population size that should be large 
enough, and, finally, an upper bound on the expected runtime is provided.

\subsection{A Uniform Bound on the Sum of Bernoulli Trials}
In order to show that conditions (G1) and (G2) in the 
level-based theorem are verified, we will use a
sharp upper bound on the probability 
$\Pr\left(Y=y\right)$ for any $y$, where $Y$ 
represents the level of a sampled offspring. 
Let $Y_i$ be a Bernoulli random variable with 
success probability $p_i$ that represents the 
bit value at position $i$ in a sampled 
offspring, and then $Y\coloneqq \sum_{i=1}^n Y_i$. 
The distribution of $Y$ is known as the 
Poisson-Binomial Distribution, and it has expectation
$\mathbb{E}\left[Y\right]=\sum_{i=1}^{n}p_i$ and 
variance $\sigma_n^2=\sum_{i=1}^{n}p_i\left(1-p_i\right)$. 
We will make use of a sharp upper bound on 
$\Pr\left(Y=y\right)$ from \cite{bib:Baillon}. 

\begin{theorem}[Theorem 2.1, \cite{bib:Baillon}]\label{thm:anticoncentration}
	Let $Y_1,Y_2,\ldots,Y_n$ be $n$ independent Bernoulli 
	random variables with success probability $p_i$. Let 
	$Y=\sum_{i=1}^{n}Y_i$ denote the sum of these random 
	variables and let $\sigma_n^2=\sum_{i=1}^{n}p_i(1-p_i)$ 
	be the variance of $S_n$. The following result holds 
	for all $n$, $y$ and $p_i$
	\begin{displaymath}
	\sigma_n\cdot\Pr\left(Y=y\right)\leq \eta	
	\end{displaymath}
	where $\eta$ is an absolute constant being
	\begin{displaymath}
	\eta =\max_{\lambda\geq 0}\sqrt{2\lambda} 
	e^{-2\lambda}\sum_{k=0}^{\infty}\left(\frac{\lambda^k}{k!}\right)^2 
	\sim 0.4688.
	\end{displaymath}
\end{theorem}

\subsection{Feige's Inequality}
To demonstrate that conditions (G1) and (G2) of the level-based
theorem hold, it is necessary to compute lower bounds on the
probability of $\prob{Y\geq y}$ where $Y$ represents the level of a
sampled individual. Following \cite{bib:Dang2015a}, we will make use of
a general result due to Feige to compute such lower bounds
\cite{bib:Feige2004} when $y<\expect{Y}$. For our purposes, it will be
convenient to use the following variant of Feige's theorem.
\begin{theorem}[Corollary 3, \cite{bib:Dang2015a}]\label{cor:feige-ineq}
	Let $Y_1,\dots,Y_n$ be $n$ independent random variables with support
	in $[0,1]$, define $Y = \sum_{i=1}^{n} Y_i$
	and $\mu = \expect{Y}$. It holds for every $\Delta>0$ that
	\begin{align*}
	\prob{Y > \mu - \Delta} \geq \min\left\{\frac{1}{13},\frac{\Delta}{1+\Delta}\right\}.
	\end{align*}
\end{theorem}

\section{Proof idea}\label{sec:proof-idea}
This section is dedicated to showing how the upper bound on the
expected runtime of \umda on \onemax is achieved using the level-based
theorem with anti-concentration bounds.  Our approach refines the
analysis in \cite{bib:Dang2015a} by taking into account
anti-concentration properties of the random variables involved.  As
already discussed in Section~\ref{level-based-theorem-section}, we
need to verify three conditions (G1), (G2) and (G3) before an
upper bound is guaranteed. The first two conditions concern the
probability of sampling an offspring belonging to a higher
level. Often verifying condition (G2) requires less effort than that
of condition (G1) since for (G2) we usually have more information on
the current population by assumption.

We chose $m'<1$, then it follows that the 
marginal probabilities are in
\begin{displaymath}
p_t(i) \in \left\{\frac{k}{\mu}\mid k\in[\mu-1]\right\}\cup
\left\{1-\frac{1}{n}, \frac{1}{n} \right\}.
\end{displaymath}
When $p_t(i)=1-1/n$ or $1/n$, we say that 
the marginal probability is at the upper
or lower border, respectively. Therefore, 
we can categorise values for $p_t(i)$ into three groups: those at the upper
margin $1-1/n$, those at the lower margin $1/n$, and those within the
closed interval $[1/\mu, 1-1/\mu]$. 
For \onemax, all bits have the same weight and the fitness 
is just the sum of these bit values, so the re-arrangement 
of bit positions will not have any impact on the 
distribution of sampled offspring. 
As a result, without loss of generality, we
can re-arrange the bit-positions so that for two integers
$k,\ell\geq 0$, it holds
\begin{itemize}
	\item for all $i\in[1,k],$ $1\leq X_i \leq \mu-1$ and $p_t(i)=X_i/\mu$, 
	\item for all $i\in(k,k+\ell]$, $X_i = \mu$ and $p_t(i)=1-1/n$, and
	\item for all $i\in(k+\ell,n]$, $X_i =0$ and $p_t(i)=1/n$.
\end{itemize}	

Given the search space $\mathcal{X}\coloneqq \{0,1\}^n$, 
we define the levels as the canonical fitness-based partition
\begin{align}
A_j & \coloneqq \left\{x\in \mathcal{X} \mid \onemax(x)=j-1\right\}.
\label{eq:level-def}
\end{align}
For a given time
$t\in\mathbb{N}$, and for all integers $i,j$ with
$1\leq i\leq j\leq n$, define the Poisson-Binomially distributed random variables
\begin{align*}
Y_{i,j} \coloneqq \sum_{k=i}^j Y_k,\quad \text{ where }\quad
Y_k \sim \bernoulli(p_t(k))\quad \text{ for all }k\in[n].
\end{align*}
Note that the probability occurring in conditions (G1) and (G2) of
the level-based theorem can now be re-written as
\begin{align*}
\Pr_{y \sim \mathcal{D}(P_t)}\left(y \in A_{\geq j+1}\right) = \prob{Y_{1,n}\geq j}.
\end{align*}

To verify condition (G1), by assumption all $\mu$ top candidate 
solutions in the current population belong to $A_{j}$, 
i.e. having exactly $j-1$ one-bits.  We need to calculate 
a lower bound $z_j$ on the probability of sampling an 
offspring having at least $j$ 1-bits. This probability 
$\Pr(Y_{1,n}\geq j)$ is the area marked 
by the diagonal lines in Figure~\ref{fig:proof}.
\begin{figure}	
	\centering
	\begin{tikzpicture}
	\begin{axis}[
	no markers, 
	domain=0:9,
	samples=100,
	axis lines*=left, 	
	xlabel=$Y$, 
	ylabel=\empty,	
	every axis y label/.style={at=(current axis.above origin),anchor=south},
	every axis x label/.style={at=(current axis.right of origin),anchor=west},	
	y axis line style={draw opacity=0},
	height=4cm, 
	width=8cm,	
	xtick={5.8,4.5},
	ytick={},
	enlargelimits=false, 
	clip=false, 
	axis on top,
	grid = major
	]
	\pgfdeclarepatternformonly{north east lines wide}%
	{\pgfqpoint{-1pt}{-1pt}}%
	{\pgfqpoint{10pt}{10pt}}%
	{\pgfqpoint{9pt}{9pt}}%
	{
		\pgfsetlinewidth{0.4pt}
		\pgfpathmoveto{\pgfqpoint{0pt}{0pt}}
		\pgfpathlineto{\pgfqpoint{9.1pt}{9.1pt}}
		\pgfusepath{stroke}
	}
	\pgfplotsset{ticks=none};
	\addplot [fill=cyan!20,pattern=north east lines wide, draw=none, domain=5.8:9] {gauss(4.5,1.3)} \closedcycle;
	\addplot [fill=black!20, draw=none, domain=4.5:5.8] {gauss(4.5,1.3)} \closedcycle;
	\addplot [very thick,cyan!50!black] {gauss(4.5,1.3)};
	\node [below] at (axis cs:  4.5,  0) {$\mathbb{E}\left[Y\right]$};
	\node [below] at (axis cs: 5.8,  0) {$j$};
	\end{axis}
	\end{tikzpicture}	
	\caption{Distribution of number of one-bits}
	\label{fig:proof}
\end{figure}
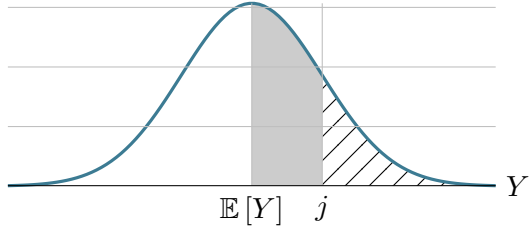

We aim to obtain an upper bound $\mathcal{O}(n\lambda)$ 
of \umda on \onemax using the level-based theorem. 
Note that the logarithmic factor $\mathcal{O}(\log\lambda)$ in the first upper 
bound $\mathcal{O}(n\lambda\log \lambda)$ in 
\cite{bib:Dang2015a} stems from the lower bound $z_j=\Omega(\mu^{-1})$.
We need a better bound $z_j=\Omega\left((n-j+1)/n\right)$. This led
us to consider three cases according to different 
configurations of the current population in Step 3 of 
Theorem~\ref{onemax-proof} below.  
\begin{enumerate}
	\item $k\geq \mu$. We will see that this implies that the variance of 
	$Y_{1,k}$ is quite large, hence the distribution of $Y_{1,k}$ cannot be
	too concentrated on the mean $\mathbb{E}[Y_{1,k}]=j-\ell-1$. 
	As a result, it is sufficient to get an extra 1-bit from the first 
	$k$ positions to obtain an offspring belonging to $A_{\geq j+1}$. 
	The probability of sampling $j$ 1-bits is bounded from below by 
	$\Pr(Y_{1,n}\geq j)\geq \Pr(Y_{1,k}\geq j-\ell)\cdot \Pr(Y_{k+1, k+\ell}=\ell)$, 
	where $\Pr(Y_{1,k}\geq j-\ell)$ is measured using the 
	anti-concentration result from Theorem~\ref{thm:anticoncentration} 
	and Lemma~\ref{int-expectation}. 
	
	\item $k<\mu$ and $j\geq n+1-n/\mu$. In this case, 
	the current level is very close to the optimal one, 
	and the bitstring has few zero-bits. 
	As already obtained from \cite{bib:Dang2015a}, 
	the upgrade probability in this case is $\Omega(\mu^{-1})$. 
	Since the condition can be rewritten as 
	$\mu^{-1}\geq (n-j+1)/n$, it ensures that 
	$z_j=\Omega(\mu^{-1})=\Omega((n-j+1)/n)$.  
	
	\item The remaining cases. Later will we prove that given
	$\mu\leq \sqrt{n(1-c)}$ for some constant $c\in (0,1)$, all
	remaining cases excluded by the first two cases are covered in
	$0\leq k<(1-c)(n-j+1)$.  In this case, $k$ is relatively small, and
	$\ell$ is not too large since the current level is not very close to
	the optimal one. This implies that most zero-bits must be located in the last
	$n-k-\ell$ positions, and it suffices to sample an extra 1-bit from
	this region.  The probability of sampling an
	offspring belonging to levels $A_{\geq j+1}$ is then $\Omega((n-j+1)/n)$.
\end{enumerate} 

\section{Runtime of UMDA on \onemax}
\onemax is the problem of maximising the number of one-bits in a
bitstring,  and is formally defined by
\begin{math}
\onemax\left(x\right)=\sum_{i=1}^{n}x_i
\end{math}.
It is well-known that the \onemax problem can be optimised in expected
time $\Theta(n\log n)$ using the $(1+1)$ Evolutionary Algorithm. 
The level-based theorem was applied to derive the first upper bound 
on the expected optimisation time of the \umda on \onemax, which is 
$\mathcal{O}(n\lambda\log \lambda)$, assuming $\lambda=\Omega(\log n)$ 
\cite{bib:Dang2015a}. By refining this method, we will obtain the better bound 
$\mathcal{O}(n\lambda)$.

\begin{theorem}\label{onemax-proof}
	For some constant $a>0$, and any constant $c\in(0,1)$,
	the UMDA with parent population size
	$a\ln(n)\leq \mu\leq \sqrt{n(1-c)}$,
	offspring population size $\lambda\geq (13e)\mu/(1-c)$,
	and margin $m'\coloneqq\mu/n$, has
	expected optimisation time $\mathcal{O}\left(n\lambda\right)$ on \onemax.
\end{theorem}
\begin{proof}
	First, we define $\gamma_0\coloneqq \mu/\lambda$. Since
	$\mu\leq \sqrt{n(1-c)}$, it 
	follows that $m'=\mu/n<1$, and the upper 
	and lower borders for $p_t(i)$ simplify to $1-1/n$ 
	and $1/n$, respectively. We re-arrange the 
	bit positions and define the random variable $Y_{i,j}$ 
	as in Section~\ref{sec:proof-idea}. We now 
	closely follow the recommended 5-step 
	procedure for applying the level-based theorem
	\cite{DBLP:Corus2017}. 
	
	\textbf{Step 1.} The levels are defined as in Eq.~(\ref{eq:level-def}).
	There are exactly $m=n+1$ levels from $A_1$ to $A_{n+1}$, 
	where level $A_{n+1}$ consists of the optimal solution. 
	
	\textbf{Step 2.} We verify condition (G2) of the level-based
	theorem. In particular, for some $\delta\in(0,1),$ and for any level 
	$j\in [m-2]$, and any  $\gamma\in (0,\gamma_0]$, assuming 
	that the population is configured  such that
	$|P_t\cap A_{\geq j}|\geq \gamma_0\lambda=\mu$ and
	$|P_t\cap A_{\geq j+1}|\geq \gamma\lambda>0$, we must show that the
	probability to sample an offspring 
	belonging to level $A_{j+1}$ or higher must be no
	less than $(1+\delta)\gamma$. By
	the re-arrangement of the bit-positions mentioned in Section \ref{sec:proof-idea}, it holds
	\begin{align}
	\sum_{i=k+1}^{k+\ell}X_i=\mu \ell \quad 
	\mbox{ and }\quad \sum_{i=k+\ell+1}^{n}X_i=0,	\label{eq:1}
	\end{align}	
	where $X_i, i\in[n],$ are given in Algorithm \ref{umda-algor}.
	By assumption, the current population $P_t$ consists of at 
	least $\gamma\lambda$ individuals with $j$ one-bits 
	and $\mu-\gamma\lambda$ individuals with $j-1$ one-bits, therefore 
	\begin{align}\label{eq:2}
	\sum_{i=1}^{n}X_i\geq \gamma\lambda j + \left(\mu-\gamma\lambda\right)(j-1) 
	= \gamma\lambda+\mu\left(j-1\right).
	\end{align}
	Combining (\ref{eq:1}), (\ref{eq:2}) and noting that $\lambda=\mu/\gamma_0$ yield
	\begin{align*}		
	\sum_{i=1}^{k}X_i 
	&  =\sum_{i=1}^{n}X_i-\sum_{i=k+1}^{k+\ell}X_i-\sum_{i=k+\ell+1}^{n}X_i \\
	& \geq \gamma\lambda+\mu\left(j-1\right)-\mu\ell=\mu\left(j-\ell-1+\frac{\gamma}{\gamma_0}\right).	
	\end{align*}
	Let $Z=Y_{1,k}+Y_{k+\ell+1,n}$ be the total number of 1-bits sampled
	in the first $k$ and the last $n-k-\ell$ positions. $Y_{1,k}$ and
	$Y_{k+\ell+1,n}$ take integer values only, and so does $Z$. Since
	$k+\ell\leq n$, the expected value of $Z$ is
	\begin{displaymath}
	\begin{split}
	\mathbb{E}\left[Z\right]
	&= \sum_{i=1}^{k}p_t(i) +\sum_{i=k+\ell+1}^{n}p_t(i) \\
	&= \frac{1}{\mu}\left(\sum_{i=1}^{k}X_i\right) +\frac{1}{n}\left(n-k-\ell\right)	
	\geq j-\ell-1+\frac{\gamma}{\gamma_0} .
	\end{split}
	\end{displaymath}	
	In order to obtain an offspring with at least $j$ one-bits, it is
	sufficient to sample $\ell$ one-bits in positions
	$k+1$ to $k+\ell$ and at least $j-\ell$ one-bits from
	the other positions. The probability of this event is bounded from below by
	\begin{align}   
	\Pr\left(Y_{1,n} \geq j\right) \geq \Pr\left(Z\geq j-\ell\right)\cdot \Pr\left(Y_{k+1,k+\ell}=\ell\right). \label{eq:g2-upgrade-prob}  
	\end{align}  	        
	The probability to obtain $\ell$ 1-bits in the middle interval from
	position $k+1$ to $k+\ell$ is 
	\begin{align}
	\Pr\left(Y_{k+1,k+\ell}=\ell\right)=\left(1-\frac{1}{n}\right)^\ell\geq \left(1-\frac{1}{n}\right)^{n-1}\geq \frac{1}{e}.\label{eq:no-mut-ell-prob}
	\end{align}	
	We now need to calculate $\Pr\left(Z\geq j-\ell\right)$. 
	Since $Z$ takes integer values only, then
	\begin{displaymath}	
	\begin{split}
	\Pr\left(Z\geq j-\ell\right)&=\Pr\left(Z> j-\ell-1\right)\\
	&\geq\Pr\left(Z>\mathbb{E}\left[Z\right]-\frac{\gamma}{\gamma_0}\right).
	\end{split}		
	\end{displaymath}		
	Applying Theorem~\ref{cor:feige-ineq} for
	$\Delta =\gamma/\gamma_0\leq 1$ and noting that we chose $\mu$ and
	$\lambda$ such that
	such that $1/\gamma_0=\lambda/\mu\geq 13e/(1-c)= 13e(1+\delta)$ yield
	\begin{align}
	\Pr\left(Z\geq j-\ell\right)
	& \geq \min\bigg\{\frac{1}{13},\frac{\Delta}{\Delta+1}\bigg\}\\
	& \geq \frac{\Delta}{13}  = \frac{\gamma}{13\gamma_0} \geq e\left(1+\delta\right)\gamma.\label{eq:g2-rest-upgrade}
	\end{align}	
	Therefore, combining (\ref{eq:g2-upgrade-prob}),
	(\ref{eq:no-mut-ell-prob}), and (\ref{eq:g2-rest-upgrade}) give
	\begin{math}
	\Pr\left(Y_{1,n}\geq j\right)\geq \left(1+\delta\right)\gamma	
	\end{math},
	and condition (G2) holds.
	
	\textbf{Step 3.} We now consider condition (G1) for any 
	level $j$ defined with $\gamma=0$. In other words, all 
	the top $\mu$ individuals in the current population $P_t$ 
	have exactly $j-1$ one-bits, and, therefore,
	\begin{math}
	\sum_{i=1}^{n}X_i= \mu\left(j-1\right)	
	\end{math}
	and
	\begin{math}
	\sum_{i=1}^{k}X_i= \mu\left(j-\ell-1\right)	
	\end{math}.	
	There are three different cases that cover all 
	situations according to variables $k$ and $j$.
	
	\underline{Case 1:} Assume that $k\geq \mu$. The variance
	of the first $k$ bits is 
	\begin{displaymath}
	\var{Y_{1,k}}=\sum_{i=1}^{k}p_t(i)\left(1-p_t(i)\right)\geq 
	\frac{k}{\mu}\left(1-\frac{1}{\mu}\right)\geq \frac{9k}{10\mu}\geq \frac{9}{10},
	\end{displaymath}
	where the second inequality holds for sufficiently large $n$ because $\mu\geq a\ln(n)$. 
	Theorem \ref{thm:anticoncentration} applied with $\sigma_k\geq
	\sqrt{9/10}$ now gives
	\begin{align*}
	\prob{Y_{1,k}=j-\ell-1}\leq \eta/\sigma_k.
	\end{align*}
	Furthermore, since  $\expect{Y_{1,k}}$ is an integer,
	Lemma~\ref{int-expectation} implies
	\begin{align*}
	\prob{Y_{1,k}\geq \expect{Y_{1,k}}} \geq 1/2.
	\end{align*}
	By combining these two probability bounds, 
	the probability to obtain at least $j-\ell$ one-bits 
	from the first $k$ positions is 
	\begin{align}
	\Pr\left(Y_{1,k}\geq j-\ell\right)
	&=\Pr\left(Y_{1,k}\geq j-\ell-1\right)-\Pr\left(Y_{1,k}=j-\ell-1\right)\nonumber\\
	&=\Pr\left(Y_{1,k}\geq \mathbb{E}\left[Y_{1,k}\right]\right)-\Pr\left(Y_{1,k}=j-\ell-1\right)\nonumber\\
	&\geq \frac{1}{2}-\frac{\eta}{\sigma_k} >\frac{1}{2}-\frac{0.4688}{\sqrt{9/10}} = \Omega(1).\label{eq:case1-pos-upgrade}
	\end{align}
	In order to obtain an offspring belonging to levels $A_{\geq j+1}$,
	it is sufficient to sample at least $j-\ell$ one-bits
	from the first $k$ positions and $\ell$ 1-bits from position
	$k+1$ to position $k+\ell$. By (\ref{eq:no-mut-ell-prob}) and 
	(\ref{eq:case1-pos-upgrade}), the probability 
	of this event is bounded from below by
	\begin{align*}
	\Pr\left(Y_{1,n}\geq j\right)&\geq \Pr\left(Y_{1,k}\geq j-\ell\right)\cdot\Pr\left(Y_{k+1,k+\ell}=\ell\right)\\
	&> \Omega(1)\cdot\frac{1}{e} =\Omega(1).
	\end{align*} 
	
	\underline{Case 2:} $k<\mu$ and $j\geq n(1-1/\mu)+1$. The 
	second condition is equivalent to $1/\mu\geq (n-j+1)/n$. The probability
	to obtain an offspring belonging to levels $A_{\geq j+1}$ is
	then bounded
	from below by
	\begin{multline*}
	\Pr\left(Y_{1,n}\geq j\right)\geq\\ 
	\Pr\left(Y_{1,1}=1\right)
	\Pr\left(Y_{2,k}\geq j-\ell-1\right)
	\Pr\left(Y_{k+1,k+\ell}=\ell\right)\\
	\geq \frac{1}{\mu}\Pr\left(Y_{2,k}\geq j-\ell-1\right)\frac{1}{e}
	\geq \frac{1}{14e\mu},
	\end{multline*}
	where we used the inequality $\Pr\left(Y_{2,k}\geq
	j-\ell-1\right)\geq 1/14$ for $\mu\geq 14$ 
	proved in \cite{bib:Dang2015a}. Since 
	$1/\mu\geq (n-j+1)/n$, we can conclude that
	\begin{displaymath}
	\Pr\left(Y_{1,n}\geq j\right)
	\geq \frac{1}{14e\mu}
	\geq \frac{n-j+1}{14en}=\Omega\left(\frac{n-j+1}{n}\right).
	\end{displaymath}	
	
	\underline{Case 3:} $k<\mu$ and $j<n(1-1/\mu)+1$. This
	case covers all the remaining situations not included by the
	first two cases. The latter inequality can be 
	rewritten as $n-j+1\geq n/\mu$. We also have 
	$\mu\leq \sqrt{n(1-c)}$, so $n/\mu\geq \mu/(1-c)$, then 
	\begin{displaymath}
	(1-c)(n-j+1)  \geq (1-c)(n/\mu) \geq(1-c)\mu/(1-c)=\mu>k.
	\end{displaymath}
	Thus, the two conditions can be shortened to 
	$0\leq k<(1-c)(n-j+1)$. In this case, 
	the probability of sampling $j$ one-bits is 
	\begin{align*}
	&\Pr(Y_{1,n}\geq j)	\\
	&\geq \Pr\left(Y_{1,k}\geq j-\ell-1\right)
	\Pr\left(Y_{k+1,k+\ell}=\ell\right)
	\Pr\left(Y_{k+\ell+1,n}\geq 1\right)\\
	&\geq \frac{1}{2}\cdot\frac{1}{e}\cdot\frac{n-k-\ell}{n} 
	= \frac{n-k-\ell}{2en}	.
	\end{align*}
	Since $\ell\leq j-1$ and $k<(1-c)(n-j+1)$, then 
	\begin{align*}
	\Pr\left(Y_{1,n}\geq j\right) > \frac{n-(1-c)(n-j+1)-j+1}{2en} 
	=\Omega\left(\frac{n-j+1}{n}\right).
	\end{align*}	
	Combining all three cases together yields the upgrade probability 
	\begin{align*}
	\Pr\left(Y_{1,n}\geq j\right) & \geq \min\bigg\{\Omega(1), \;
	\Omega\left(\frac{n-j+1}{n}\right)\bigg\}
	= \Omega\left(\frac{n-j+1}{n}\right) =: z_j,
	\end{align*}
	and, therefore, $z_*\coloneqq \min_{j\in[n]}\{z_j\}=\Omega(1/n)$.
	
	\textbf{Step 4.} We consider condition (G3) regarding the 
	population size. We have $1/\delta^2=\mathcal{O}(1)$,
	$1/z_*= \mathcal{O}(n)$, and $m=\mathcal{O}(n)$. Therefore there must 
	exist a constant $a>0$ such that	    
	\begin{align*}
	\left(\frac{a}{\gamma_0}\right)\ln(n) \geq \left(\frac{4}{\gamma_0\delta^2}\right)\ln\left(\frac{128m}{z_*\delta^2}\right). 
	\end{align*}
	The requirement $\mu\geq a\ln(n)$ now implies that
	\begin{align*}
	\lambda 
	=  \frac{\mu}{\mu/\lambda} 
	\geq \left(\frac{a}{\gamma_0}\right)\ln(n)
	\geq \left(\frac{4}{\gamma_0\delta^2}\right)\ln\left(\frac{128m}{z_*\delta^2}\right),
	\end{align*}
	hence condition (G3) is satisfied. 
	
	\textbf{Step 5.} We have shown that conditions (G1), (G2), and (G3) are satisfied. By
	Theorem~\ref{thm:levelbasedtheorem} and the bound $z_j=\Omega((n-j+1)/n)$,
	the expected optimisation time is therefore
	\begin{align*}
	\mathbb{E}\left[T\right]=
	\mathcal{O}\left(\lambda\sum_{j=1}^{n}\ln\left(\frac{n}{n-j+1}\right)+
	\sum_{j=1}^{n}\frac{n}{n-j+1}\right).
	\end{align*}		
	We now estimate the two terms separately. By  
	Stirling's approximation (Lemma~\ref{stirling}), the first term is
	\begin{align*}
	\mathcal{O}\left(\lambda\sum_{j=1}^{n}\ln\left(\frac{n}{n-j+1}\right)\right)
	&=\mathcal{O}\left(\lambda\ln\prod_{j=1}^n\frac{n}{n-j+1}\right)
	=\mathcal{O}\left(\lambda\ln \frac{n^n}{n!}\right)\\
	&    =\mathcal{O}\left(\lambda\ln \frac{n^n\cdot e^n}{n^{n+1/2}}\right)
	= \mathcal{O}\left(n\lambda\right).
	\end{align*}
	The second term is
	\begin{displaymath}
	\mathcal{O}\left(\sum_{j=1}^{n}\frac{n}{n-j+1}\right)
	=\mathcal{O}\left(n\sum_{k=1}^{n}\frac{1}{k}\right) 
	= \mathcal{O}\left(n\log n\right).	
	\end{displaymath}
	Since $\lambda > \mu= \Omega(\log n)$,
	the expected optimisation time is 
	\begin{displaymath}
	\mathbb{E}\left[T\right]=\mathcal{O}\left(n\lambda\right)+
	\mathcal{O}\left(n\log n\right)=\mathcal{O}\left(n\lambda\right).
	\end{displaymath}
\end{proof}	

\section{An empirical result}
So far we have proven an upper bound 
$\mathcal{O}\left(n\lambda\right)$ on the expected 
runtime of \umda on \onemax with  
parent population size $a\log n\leq \mu=\mathcal{O}(\sqrt{n})$ , offspring population size $\lambda = \Omega(\mu)$,
and margin size $m'\leq 1$. This result is tighter than the 
bound $\mathcal{O}(n\lambda\log \lambda)$, 
obtained in \cite{bib:Dang2015a}, which provided the first
upper bound for \umda on \onemax. However, the bound $\mathcal{O}(n\lambda)$ 
is asymptotic and only provides 
information on the growth of the expected runtime 
according to the problem size $n$ for 
sufficiently large $n \geq n_0$. It provides no information on 
the multiplicative constant or the 
influences of lower order terms. Hence it 
makes sense to consider the empirical runtime of 
\umda on \onemax to partially compensate for the limitations in the
theoretical analysis.

We carry out a small experiment by running the \umda on \onemax 
with initial parameter settings consistent with those 
conditions mentioned above. 
The settings of parameters are as follows: 
$\lambda=\sqrt{n}$, 
$\mu=\log n$ and $m'=0.5$ 
for $n \in \{100, 200,\ldots,10000\}$. 
The results are shown in Figure \ref{fig:exp1-plot}. 
For each value of $n$, the algorithm is run 100 
times, and then the average runtime is computed. 
The mean runtime for each value of $n$ is estimated
with 95\% confidence intervals using the
\textit{bootstrap percentile method} \cite{bib:Lehre2014} with 100 
bootstrap samples. Each mean point is plotted 
with two error bars to illustrate the upper and 
lower margins of the confidence intervals.

\begin{figure}
\centering
	\includegraphics[width=.495\textwidth]{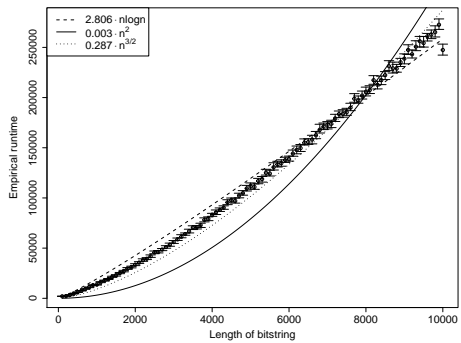}
	\caption{Mean runtime of \umda on \onemax with 95\% 
		confidence intervals plotted with error bars. 
		The fitted models are also plotted.}
	\label{fig:exp1-plot}
\end{figure}

From the parameter settings chosen for the experiment, 
Theorem~\ref{onemax-proof} gives the upper bound 
$\mathcal{O}(n^{3/2})$ for the expected optimisation time. We now 
compare this theoretical bound with the empirical runtime 
and two other bounds close to this model:
$\mathcal{O}(n\log n)$ which is the runtime of 
(1+1) \ea on \onemax, and the quadratic bound $\mathcal{O}(n^2)$. 
Following \cite{bib:Lehre2014}, we fit three positive constants 
$c_1, c_2$ and $c_3$ to the models 
$c_1\cdot n\log n$, $c_2\cdot n^{3/2}$ and $c_3\cdot n^{2}$ 
using non-linear 
least square regression. The 
correlation coefficient for each model is calculated to 
measure the fit of each model to the data. 

\begin{table}
	\centering
	\caption{Best-fit models.}
	\label{tab:exp1-table}
	\begin{tabular}{@{}l l@{}}
		\toprule
		Best-fit function &  Correlation coefficient \\
		\midrule	
		$2.806 \cdot n\log n$ & 0.9994 \\
		$0.287 \cdot n^{3/2}$ & 0.9900\\
		$0.003 \cdot n^{2}$ & 0.9689\\
		\bottomrule
	\end{tabular}
\end{table}	

From Table \ref{tab:exp1-table}, it can be seen that the first two
models $2.806\cdot n\log n$ and $0.287\cdot n^{3/2}$, with the
correlation coefficients $0.9994$ and $0.9900$ respectively, fit well
with the empirical data. The quadratic model fits less well with the
empirical data. These findings are consistent with the
theoretical expected optimisation time since the first two
models are members of $\mathcal{O}(n^{3/2})$. 
As already stated before, our bound $\mathcal{O}(n\lambda)$ 
is tight for $\lambda=\mathcal{O}(\log n)$; 
however, in this experiment we chose a larger offspring population size $\lambda=\sqrt{n}$.
For this case, the model $2.806\cdot n\log n$ has 
higher correlation coefficient than the model $0.287\cdot n^{3/2}$,
indicating that our theoretical bound may not be tight for this case.

\section{Conclusion}
Despite the long-time use of EDAs by the Evolutionary Computation
community, little has been known about their runtime, even for
apparently simple settings such as \umda on \onemax. Results about
the UMDA are not only relevant to Evolutionary Computation, but also
to Population Genetics where it corresponds to the notion of
\emph{linkage equilibrium}.

We have proved the upper bound $\mathcal{O}(n\lambda)$ which holds for
$a\log n\leq \mu=\mathcal{O}(\sqrt{n})$ where $a$ is a positive
constant. Although our result assumes that $\lambda\geq (1+c')\mu$ for some
positive constant $c'>0$, it does not require that $\mu$ is proportional in
size to $\lambda$.  The bound is tight when
$\lambda =\mathcal{O}(\log n)$; in this case, a tight bound
$\Theta(n\log n)$ on the expected optimisation time of the \umda on
\onemax is obtained, matching the well-known bound $\Theta(n\log n)$
for the (1+1) \ea on the class of linear functions. Although the bound
assumes a not too large parent population size
$\mu=\mathcal{O}(\sqrt{n})$, it finally closes the $\Theta(\log \log
n)$  gap
between the first upper bound
$\mathcal{O}(n\log n \log \log n)$ \cite{bib:Dang2015a} for certain
settings of $\lambda$ and $\mu$ and the recently discovered lower
bound $\Omega(\mu\sqrt{n}+n\log n)$ for $\lambda = (1+\Theta(1))\mu$
\cite{bib:Krejca}.  Future work should consider the runtime of UMDA on
\onemax for larger offspring population sizes $\mu=\omega(\sqrt{n})$
and different combinations of $\mu$ and $\lambda$,
as well as the runtime on more complex fitness landscapes.

Our analysis further demonstrates that the level-based theorem can
yield, relatively easily, asymptotically tight bounds for non-trivial,
population-based algorithms. An important additional component of the
analysis was the use of anti-concentration properties of the
Poisson-Binomial distribution. Unless the variance of the sampled
individuals is not too small, the distribution of the population
cannot be too concentrated anywhere, yielding sufficient diversity to
discover better solutions. We expect that these arguments will lead to
new results in runtime analysis of evolutionary algorithms.
\appendix

We use the following property of the
Poisson-Binomial distribution.

\begin{lemma}[Theorem 3.2, \cite{bib:Jogdeo}]\label{int-expectation}
	Let $Y_1,Y_2,\ldots,Y_n$ be $n$ independent Bernoulli random variables. 
	Let $Y\coloneqq \sum_{i=1}^{n}Y_i$ be the sum of these random variables and let 
	$\mu$ be the expectation of $Y$. If $\mu$ is an integer, then 
	\begin{displaymath}
	\Pr\left(Y\geq \mu\right)\geq 1/2.
	\end{displaymath}
\end{lemma}

\begin{lemma}[Stirling's approximation \cite{bib:LCRC}]
	\label{stirling} For all $n\in \mathbb{N}$,
	\begin{displaymath}
	n!=\Theta\left(\frac{n^{n+1/2}}{e^n}\right).
	\end{displaymath}
\end{lemma}

\end{document}